\documentclass[twoside]{article}

\usepackage{aistats2019}
\usepackage[round]{natbib}
\usepackage{math} 						

\usepackage{tikz}
\usetikzlibrary{arrows,calc,shapes,intersections}

\definecolor{C0}{HTML}{1f77b4}
\definecolor{C1}{HTML}{ff7f0e}
\definecolor{C2}{HTML}{2ca02c}
\definecolor{C3}{HTML}{d62728}
\definecolor{C4}{HTML}{9467bd}
\definecolor{C5}{HTML}{8c564b}
\definecolor{C6}{HTML}{e377c2}
\definecolor{C7}{HTML}{7f7f7f}
\definecolor{C8}{HTML}{bcbd22}
\definecolor{C9}{HTML}{17becf}

\usepackage{math}
\usepackage{enumitem}
\usepackage{xspace}
\usepackage{tabularx}
\usepackage{booktabs} 

\usepackage{mathtools}

\newcommand{\MATLAB}{\textsc{Matlab}\xspace}

\newcommand{\bx}{\mathbf{x}}
\newcommand{\bu}{\mathbf{u}}
\newcommand{\bp}{\mathbf{p}}
\newcommand{\bw}{\mathbf{w}}
\newcommand{\bz}{\mathbf{z}}

\begin{document}

%

%

\twocolumn[

\aistatstitle{An Optimal Control Approach to Sequential Machine Teaching}
\aistatsauthor{ Laurent Lessard \And Xuezhou Zhang \And  Xiaojin Zhu }
\aistatsaddress{ University of Wisconsin--Madison \And  University of Wisconsin--Madison \And University of Wisconsin--Madison } ]

\begin{abstract}

Given a sequential learning algorithm and a target model, sequential machine teaching aims to find the shortest training sequence to drive the learning algorithm to the target model. We present the first principled way to find such shortest training sequences. Our key insight is to formulate sequential machine teaching as a time-optimal control problem. This allows us to solve sequential teaching by leveraging key theoretical and computational tools developed over the past 60 years in the optimal control community. Specifically, we study the Pontryagin Maximum Principle, which yields a necessary condition for optimality of a training sequence. We present analytic, structural, and numerical implications of this approach on a case study with a least-squares loss function and gradient descent learner. We compute optimal training sequences for this problem, and although the sequences seem circuitous, we find that they can vastly outperform the best available heuristics for generating training sequences.
\end{abstract}

\section{INTRODUCTION}

Machine teaching studies optimal control on machine learners~\citep{Zhu2018Overview,Zhu2015Machine}.
In controls language the plant is the learner, the state is the model estimate, and the input is the (not necessarily $i.i.d.$) training data.
The controller wants to use the least number of training items---a concept known as the teaching dimension~\citep{Goldman1995Complexity}---to force the learner to learn a target model.
For example, in adversarial learning, an attacker may minimally poison the training data to force a learner to learn a nefarious model~\citep{biggio12-icml,Mei2015Machine}.
Conversely, a defender may immunize the learner by injecting adversarial training examples into the training data~\citep{2014arXiv1412.6572G}.
In education systems, a teacher may optimize the training curriculum to enhance student (modeled as a learning algorithm) learning~\citep{Sen2018Machine,Patil2014Optimal}.

Machine teaching problems are either batch or sequential depending on the learner.
The majority of prior work studied batch machine teaching, where the controller performs one-step control by giving the batch learner an input training \emph{set}.
Modern machine learning, however, extensively employs sequential learning algorithms.
We thus study sequential machine teaching: what is the shortest training sequence to force a learner to go from an initial model $\bw_0$ to some target model $\bw_\star$?
Formally, at time $t=0,1,\ldots$ the controller chooses input $(\bx_t, y_t)$ from an input set $\mathcal U$. The learner then updates the model according to its learning algorithm.  This forms a dynamical system $f$:
\begin{subequations}\label{eq:regression}
	\begin{equation}
	\bw_{t+1} = f(\bw_t, \bx_t, y_t).
	\end{equation}
	The controller has full knowledge of $\bw_0, \bw_\star, f, \mathcal U$, and wants to minimize the terminal time $T$ subject to $\bw_T=\bw_\star$.
	
	As a concrete example, we focus on teaching a gradient descent learner of least squares: 
	\begin{equation}
	f(\bw_t, \bx_t, y_t) = \bw_t - \eta (\bw_t^\tp \bx_t - y_t)\bx_t  
	\end{equation}
\end{subequations}
with $\bw \in \R^n$ and the input set $\norm{\bx}\le R_x, |y|\le R_y$.
We caution the reader not to trivialize the problem: \eqref{eq:regression} is a \emph{nonlinear} dynamical system due to the interaction between $\bw_t$ and $\bx_t$. 
A previous best attempt to solve this control problem by \cite{liu2017iterative} employs a greedy control policy, which at step $t$ optimizes $\bx_t, y_t$ to minimize the distance between $\bw_{t+1}$ and $\bw_\star$. 
One of our observations is that this greedy policy can be substantially suboptimal. 
Figure~\ref{fig:optvsgreedy} shows three teaching problems and the number of steps $T$ to arrive at $\bw_\star$ using different methods.
Our optimal control method NLP found shorter teaching sequences compared to the greedy policy (lengths 151, 153, 259 for NLP vs 219, 241, 310 for GREEDY, respectively).
This and other experiments are discussed in Section~\ref{sec:numerical}.
\begin{figure*}[ht]
	\includegraphics{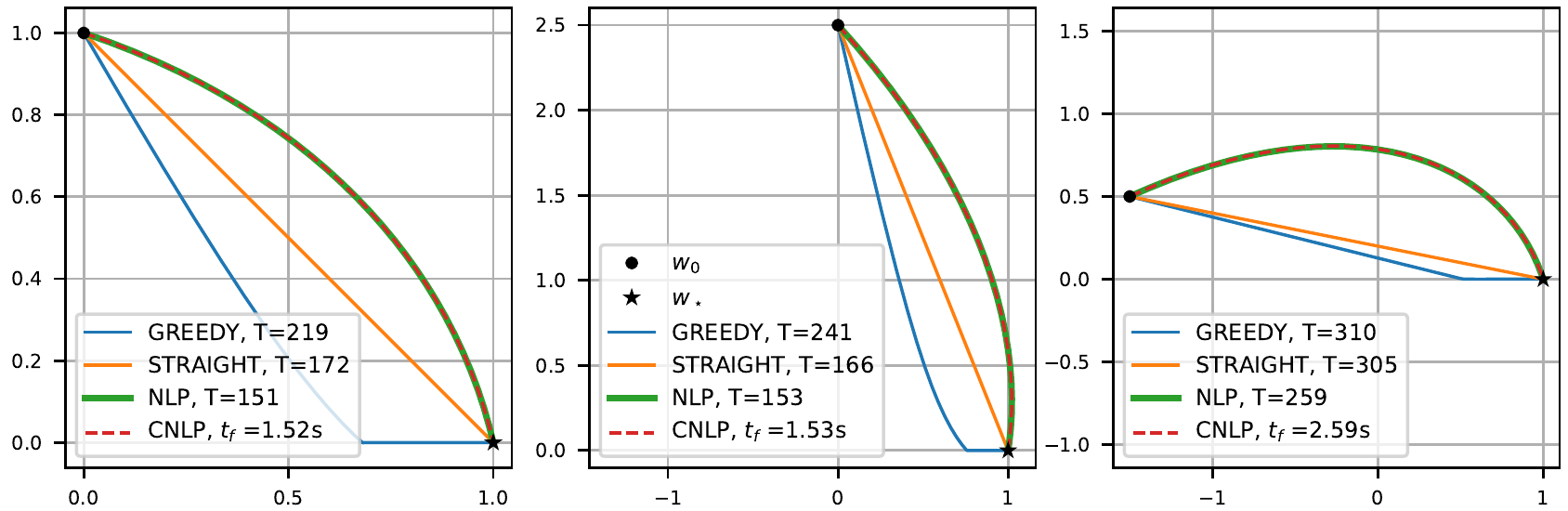}\vspace{-2mm}
	\caption{The shortest teaching trajectories found by different methods.  
		All teaching tasks use the terminal point $\bw_\star=(1,0)$. The initial points used are $\bw_0 = (0,1)$ (left panel), $\bw_0 = (0,2.5)$ (middle panel), and $\bw_0 = (-1.5,0.5)$ (right panel).
		The learner is the least squares gradient descent algorithm~\eqref{eq:regression} with $\eta=0.01$ and $R_x=R_y=1$. 
		Total steps $T$ to arrive at $\bw_\star$ is indicated in the legends.\label{fig:optvsgreedy}}\vspace{-2mm}
\end{figure*}

\newpage
\subsection{Main Contributions}

\vspace{-1mm}
Our main contribution is to show how tools from optimal control theory may be brought to bear on the machine teaching problem. Specifically, we show that:
\begin{enumerate}[itemsep=1mm,topsep=1mm,parsep=1mm]
	\item The Pontryagin optimality conditions reveal deep structural properties of optimal teaching sequences. For example, we show that the least-squares case~\eqref{eq:regression} is fundamentally a 2D problem and we provide a structural characterization of solutions. These results are detailed in Section~\ref{sec:example}.
	
	\item Optimal teaching sequences can be vastly more efficient than what may be obtained via common heuristics. We present two optimal approaches: an exact method (NLP) and a continuous approximation (CNLP). Both agree when the stepsize $\eta$ is small, but CNLP is more scalable because its runtime does not depend on the length of the training sequence. These results are shown in Section~\ref{sec:numerical}.
\end{enumerate}

We begin with a survey of the relevant optimal control theory and algorithms literature in Section~\ref{sec:time-optimal-control}.

\section{TIME-OPTIMAL CONTROL} \label{sec:time-optimal-control}
To study the structure of optimal control we consider the continuous \textit{gradient flow} approximation of gradient descent, which holds in the limit $\eta \rightarrow 0$.
In this section, we present the corresponding canonical time-optimal control problem and summarize some of the key theoretical and computational tools that have been developed over the past 60 years to address it. For a more detailed exposition on the theory, we refer the reader to modern references on the topic \citep{kirk,liberzon,athansfalb}.

This section is self-contained and we will use notation consistent with the control literature ($x$ instead of $\bw$, $u$ instead of $(\bx,y)$, $t_f$ instead of $T$).
We revert back to machine learning notation in section~\ref{sec:example}.
Consider the following boundary value problem:
\begin{equation}\label{eq:oc_ODE}
\dot{x} = f(x,u)
\qquad\text{with }x(0)=x_0\text{ and }x(t_f)=x_f.
\end{equation}
The function $x:\R_+\to\R^n$ is called the \textit{state} and $u:\R_+ \to \mathcal{U}$ is called the \textit{input}. Here, $\mathcal{U}\subseteq \R^m$ is a given constraint set that characterizes admissible inputs. The initial and terminal states $x_0$ and $x_f$ are fixed, but the terminal time $t_f$ is free. If an admissible $u$ together with a state $x$ satisfy the boundary value problem~\eqref{eq:oc_ODE} for some choice of $t_f$, we call $(x,u)$ a \textit{trajectory} of the system. The objective in a time-optimal control problem is to find an \textit{optimal trajectory}, which is a trajectory that has minimal $t_f$.

Established approaches for solving time-optimal control problems can be grouped in three broad categories: dynamic programming, indirect methods, and direct methods. We now summarize each approach.

\subsection{Dynamic Programming}\label{sec:dynprog}

Consider the value function $V:\R^n \to \R_+$, where $V(x)$ is the minimum time required to reach $x_f$ starting at the initial state $x$. The Hamilton--Jacobi--Bellman (HJB) equation gives necessary and sufficient conditions for optimality and takes the form:
\begin{equation}\label{eq:HJB}
\min_{\tilde u \in \mathcal{U}}\,\, \grad V(x)^\tp f(x,\tilde u) + 1 = 0
\qquad\text{for all }x\in\R^n
\end{equation}
together with the boundary condition $V(x_f) = 0$. If the solution to this differential equation is $V_\star$, then the optimal input is given by the minimizer:
\begin{equation}\label{eq:HJB2}
u(x) \in \argmin_{\tilde u \in \mathcal{U}}\,\, \grad V_\star(x)^\tp f(x,\tilde u) 
\quad\text{for all }x\in\R^n
\end{equation}
A nice feature of this solution is that the optimal input $u$ depends on the current state $x$. In other words, HJB produces an optimal \textit{feedback policy}.

Unfortunately, the HJB equation~\eqref{eq:HJB} is generally difficult to solve. Even if the minimization has a closed form solution, the resulting differential equation is often intractable. We remark that the optimal $V_\star$ may not be differentiable. For this reason, one looks for so-called \textit{viscosity solutions}, as described by \citet{liberzon,tonon} and references therein.

Numerical approaches for solving HJB include the fast-marching method \citep{tsitsiklis1995efficient} and Lax--Friedrichs sweeping~\citep{kao2004lax}. The latter reference also contains a detailed survey of other numerical schemes.

\subsection{Indirect Methods} \label{sec:indirect}

Also known as ``optimize then discretize'', indirect approaches start with necessary conditions for optimality obtained via the Pontryagin Maximum Principle (PMP). The PMP may be stated and proved in several different ways, most notably using the Hamiltonian formalism from physics or using the calculus of variations. Here is a formal statement.

\begin{thm}[PMP] \label{thm:PMP}
	Consider the boundary value problem~\eqref{eq:oc_ODE} where $f$ and its Jacobian with respect to $x$ are continuous on $\R^n\times \mathcal{U}$. Define the \textit{Hamiltonian} $H:\R^n\times \R^n\times \mathcal{U} \to \R$ as
	$
	H(x,p,u) \defeq p^\tp f(x,u) + 1
	$.
	If $(x^\star,u^\star)$ is an optimal trajectory, then there exists some function $p^\star :\R_+ \to \R^n$ (called the ``co-state'') such that the following conditions hold.
	\begin{enumerate}[label={\alph*}),itemsep=1mm,topsep=1mm,parsep=1mm]
		\item $x^\star$ and $p^\star$ satisfy the following system of differential equations for $t \in [0,t_f]$ with boundary conditions $x^\star(0)=x_0$ and $x^\star(t_f) = x_f$.
		\begin{subequations}\label{PMPa}
			\begin{align}
			\dot x^\star(t) &= \frac{\partial H}{\partial p}\bigl( x^\star(t),p^\star(t),u^\star(t) \bigr), \\
			\dot p^\star(t) &=  -\frac{\partial H}{\partial x}\bigl( x^\star(t),p^\star(t),u^\star(t) \bigr).
			\end{align}
		\end{subequations}
		\item For all $t\in[0,t_f]$, an optimal input $u^\star(t)$ satisfies:
		\begin{equation}\label{PMPb}
		u^\star(t) \in \argmin_{\tilde u \in \mathcal{U}} \,\,
		H(x^\star(t),p^\star(t),\tilde u).
		\end{equation}
		\item Zero Hamiltonian along optimal trajectories:
		\begin{equation}\label{PMPc}
		H(x^\star(t),p^\star(t),u^\star(t))=0
		\quad\text{for all }t\in[0,t_f].
		\end{equation}
	\end{enumerate}
\end{thm}
In comparison to HJB, which needs to be solved for all $x\in\R^n$, the PMP only applies along optimal trajectories. Although the differential equations~\eqref{PMPa} may still be difficult to solve, they are simpler than the HJB equation and therefore tend to be more amenable to both analytical and numerical approaches. Solutions to HJB and PMP are related via $\grad V^\star(x^\star(t)) = p^\star(t)$.

PMP is only necessary for optimality, so solutions of~\eqref{PMPa}--\eqref{PMPc} are not necessarily optimal. Moreover, PMP does not produce a feedback policy; it only produces optimal trajectory \textit{candidates}. Nevertheless, PMP can provide useful insight, as we will explore in Section~\ref{sec:example}.

If PMP cannot be solved analytically, a common numerical approach is the \textit{shooting method}, where we guess $p^\star(0)$, propagate the equations~\eqref{PMPa}--\eqref{PMPb} forward via numerical integration. Then $p^\star(0)$ is refined and the process is repeated until the trajectory reaches $x_f$.

\subsection{Direct Methods}\label{sec:direct}

Also known as ``discretize then optimize'', a sparse nonlinear program is solved, where the variables are the state and input evaluated at a discrete set of timepoints. An example is \textit{collocation methods}, which use different basis functions such as piecewise polynomials to interpolate the state between timepoints. For contemporary surveys of direct and indirect numerical approaches, see \citet{rao2009survey,betts2010practical}.

If the dynamics are already discrete as in~\eqref{eq:regression}, we may directly formulate a nonlinear program. We refer to this approach as NLP. Alternatively, we can take the continuous limit and then discretize, which we call CNLP. We discuss the advantages and disadvantages of both approaches in Section~\ref{sec:numerical}.

\section{TEACHING LEAST SQUARES: INSIGHT FROM PONTRYAGIN} \label{sec:example}

In this section, we specialize time-optimal control to least squares.
To recap, our goal is to find the minimum number of steps $T$ such that there exists a control sequence $(\bx_t,y_t)_{0:T-1}$ that drives the learner~\eqref{eq:regression} with initial state  $\bw_0$ to the target state $\bw_\star$.
The constraint set is $\mathcal{U} = \set{(\bx,y)}{\norm{\bx}\leq R_x,\, |y|\leq R_y}$. 
This is an \emph{nonlinear discrete-time time-optimal control problem}, for which no closed-form solution is available.

On the corresponding continuous-time control problem, applying Theorem~\ref{thm:PMP} we obtain the following necessary conditions for optimality\footnote{State, co-state, and input in Theorem~\ref{thm:PMP} are $(x,p,u)$, which is conventional controls notation. For this problem, we use $(\bw,\bp,(\bx,y))$, which is machine learning notation.} for all $t\in [0,t_f]$.
\begin{subequations}\label{eq:PMP_LS}
	\begin{align}
	\bw(0) &= \bw_0, \quad \bw(t_f) = \bw_\star \\
	\dot\bw(t) &= \bigl( y(t)-\bw(t)^\tp\bx(t) \bigr)\, \bx(t) \label{PMP_QCQP_w}\\
	\dot\bp(t) &= \bigl( \bp(t)^\tp\bx(t) \bigr)\, \bx(t)\label{PMP_QCQP_p} \\
	\bx(t),y(t) &\in \!\!\argmin_{\norm{\hat\bx} \le R_x,\, |\hat y| \le R_y } \!\bigl( \hat y-\bw(t)^\tp \hat\bx \bigr)(\bp(t)^\tp  \hat\bx) \label{PMP_QCQP_u}\\
	0 &= \bigl( y(t)-\bw(t)^\tp \bx(t) \bigr) \bigl( \bp(t)^\tp \bx(t) \bigr) + 1
	\label{PMP_QCQP_eqn}
	\end{align}
\end{subequations}

We can simplify~\eqref{eq:PMP_LS} by setting $y(t) = R_y$, as described in Proposition~\ref{prop:y_simplify} below.

\begin{prop}\label{prop:y_simplify}
	For any trajectory $(\bw,\bp,\bx,y)$ satisfying~\eqref{eq:PMP_LS}, there exist another trajectory of the form $(\bw,\bp,\tilde\bx,R_y)$. So we may set $y(t)=R_y$ without any loss of generality.
\end{prop}

\begin{proof}
	Since~\eqref{PMP_QCQP_u} is linear in $\hat y$, the optimal $\hat y$ occurs at a boundary and $\hat y = \pm R_y$. Changing the sign of $\hat y$ is equivalent to changing the sign of $\hat\bx$, so we may assume without loss of generality that $\hat y = R_y$. These changes leave~\eqref{PMP_QCQP_w}--\eqref{PMP_QCQP_p} and \eqref{PMP_QCQP_eqn} unchanged so $\bw$ and $\bp$ are unchanged as well.
\end{proof}

In fact, Proposition~\ref{prop:y_simplify} holds if we consider trajectories of~\eqref{eq:regression} as well. For a proof, see the appendix.

Applying Proposition~\ref{prop:y_simplify}, the conditions~\eqref{PMP_QCQP_u} and \eqref{PMP_QCQP_eqn} may be combined to yield the following quadratically constrained quadratic program (QCQP) equation.
\begin{eqnarray}\label{eq:QCQP}
\min_{\norm{\bx}\le R_x} (R_y-\bw^\tp \bx)(\bp^\tp  \bx) = -1
\end{eqnarray}
where we have omitted the explicit time specification $(t)$ for clarity.
Note that~\eqref{eq:QCQP} constrains the possible tuples $(\bw,\bp,\bx)$ that can occur as part of an optimal trajectory. So in addition to solving the left-hand side to find $\bx$, we must also ensure that it's equal to $-1$. We will now characterize the solutions of~\eqref{eq:QCQP} by examining five distinct regimes of the solution space that depend on the relationship between $\bw$ and $\bp$ as well as which regime transitions are admissible.

\paragraph{Regime I (Origin): $\bw =0$ and $\bp \ne 0$.} This regime happens when the teaching trajectory pass through the origin. In this regime, one can obtain closed-form solutions. In particular, $\bx = -\frac{R_x}{\norm{\bp}}\bp$ and $\norm{\bp} = \tfrac{1}{R_xR_y}$. In this regime, both $\dot{\bw}$ and $\dot{\bp}$ are positively aligned with $\bp$. Therefore, Regime I necessarily \textit{transitions} from Regime II and into Regime III, given that it is not at the beginning or the end of the teaching trajectory.

\paragraph{Regime II (positive alignment): $\bw = \alpha \bp$ with $\bp \ne 0$ and $\alpha > 0$.} This regime happens when $\bw$ and $\bp$ are positively aligned. Again we have closed form solutions. In particular, $\bx^\star = -\tfrac{R_x}{\norm{\bw}} \bw$ and $\alpha = R_x\norm{\bw}(R_y+R_x\norm{\bw})$. In this regime, both $\dot{\bw}$ and $\dot{\bp}$ are negatively aligned with $\bw$, thus Regime II necessarily transitions into Regime I and can never transition from any other regimes.

\paragraph{Regime III (negative alignment inside the origin-centered ball): $\bw = -\alpha \bp$ with $\bp \ne 0$ and $\alpha > 0$ and $\norm{\bw}\leq \frac{R_y}{2R_x}$.} This regime happens when $\bw$ and $\bp$ are negatively aligned and $\bw$ is inside the ball centered at the origin with radius $R=\frac{R_y}{2R_x}$. Again, closed form solutions exists: $\bx^\star = \tfrac{R_x}{\norm{\bw}} \bw$ and $\alpha = R\norm{\bw}(1-R\norm{\bw})$. Regime III necessarily transitions from Regime I and into Regime IV.

\paragraph{Regime IV (negative alignment out of the origin-centered ball): $\bw = -\alpha \bp$ with $\bp \ne 0$ and $\alpha > 0$ and $\norm{\bw} > \frac{R_y}{2R_x}$.} In this case, the solutions satisfies $\alpha = \tfrac{R_y^2}{4}$ so that $\bp$ is uniquely determined by $\bw$. However, the optimal $\bx^\star$ is \textbf{not} unique. Any solution to $\bw^\tp \bx = \tfrac{R_y}{2}$ with $\norm{\bx} \le R_x$ can be chosen. Regime IV can only transition from Regime III and cannot transition into any other regime. In other word, once the teaching trajectory enters Regime IV, it cannot escape. Another interesting property of Regime IV is that we know exactly how fast the norm of $\bw$ is changing. In particular, knowing $\bw^\tp \bx = \tfrac{R_y}{2}$, one can derive that $\frac{\mathrm{d}\norm{\bw}^2}{\mathrm{d}t} = \frac{R_y^2}{2}$. As a result, once the trajectory enters regime IV, we know exact how long it will take for the trajectory to reach $\bw_\star$, if it is able to reach it.

\paragraph{Regime V (general positions): $\bw$ and $\bp$ are linearly independent.} This case covers the remaining possibilities for the state and co-state variables. To characterize the solutions in this regime, we'll first introduce some new coordinates.
Define $\{\hat\bw,\hat\bu\}$ to be the orthonormal basis for $\mathrm{span}\{\bw,\bp\}$ such that $\bw = \gamma \hat \bw$ and $\bp = \alpha\hat \bw + \beta\hat \bu$ for some $\alpha,\beta,\gamma\in\R$. Note that $\beta\ne 0$ because we assume $\bw$ and $\bp$ are assumed to be linearly independent in this regime.
We can therefore express any input uniquely as $\bx = w \hat\bw + u\hat\bu + z \hat \bz$ where $\hat\bz$ is an \textit{out-of-plane} unit vector orthogonal to both $\hat\bw$ and $\hat\bu$, and $w,u,z\in\R$ are suitably chosen. Substituting these definitions, \eqref{eq:QCQP} becomes
\begin{equation}\label{eq:QCQP_simpler}
\underset{w^2+u^2+z^2 \le R_x^2}{\min}\quad (R_y - \gamma w)( \alpha w + \beta u) = -1.
\end{equation}
Now observe that the objective is linear in $u$ and does not depend on $z$. The objective is linear in $u$ because $\beta\ne 0$ and $(1-\gamma w)\ne 0$ otherwise the entire objective would be zero. Since the feasible set is convex, the optimal $u$ must occur at the boundary of the feasible set of variables $w$ and $u$. Therefore, $z=0$. This is profound, because it implies that in Regime~V, the optimal solution necessarily lies on the 2D plane $\mathrm{span}\{\bw,\bp\}$. In light of this fact, we can pick a more convenient parametrization. Let $w = R_x\cos\theta$ and $u = R_x\sin\theta$. Equation~\eqref{eq:QCQP_simpler} becomes:
\begin{eqnarray}
\min_\theta\quad R_x(R_y-\gamma R_x \cos\theta)(\alpha \cos\theta + \beta\sin\theta ) = -1.
\end{eqnarray}
This objective function has at most four critical points, of which there is only one global minimum, and we can find it numerically. Last but not least, Regime V does not transition from or into any other Regime. 

\begin{figure}[th]
	\centering
	\begin{tikzpicture}
	\tikzstyle{n} = [very thick,circle,inner sep=0mm,minimum width=7mm]
	\tikzstyle{a} = [thick,>=latex,->]
	\def\dx{1.2}
	\def\dy{0.4}
	\node[n,C1,draw=C1] (2) at (0,\dy) {\textbf{\textsf{II}}};
	\node[n,black,draw=black] (1) at (\dx,0) {\textbf{\textsf{I}}};
	\node[n,C3,draw=C3] (3) at (2*\dx,0) {\textbf{\textsf{III}}};
	\node[n,C4,draw=C4] (4) at (3*\dx,0) {\textbf{\textsf{IV}}};
	\node[n,C0,draw=C0] (5) at (5*\dx,0) {\textbf{\textsf{V}}};
	\node[n] (ws) at (4*\dx,\dy) {$\bw_\star$};
	\path[a]
	(2) edge [loop below] (2)
	(2) edge [bend right=10] (1)
	(2) edge [bend left=10] (ws)
	(1) edge (3)
	(3) edge [loop below] (2)
	(3) edge (4)
	(4) edge [loop below] (4)
	(4) edge [bend right=10] (ws)
	(5) edge [loop below] (5)
	(5) edge [bend left=10] (ws);
	\end{tikzpicture}
	\includegraphics{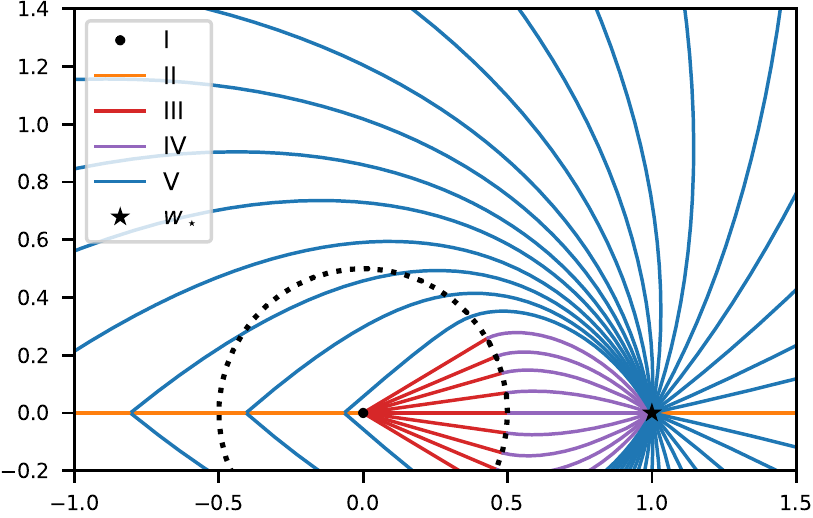}
	\caption{Optimal trajectories for $\bw_\star = (1,0)$ for different choices of $\bw_0$. Trajectories are colored according to the regime to which they belong and the directed graph above shows all possible transitions. The optimal trajectories are symmetric about the $x$-axis. For implementation details, see Section~\ref{sec:numerical}.
		\label{fig:vectorfield}}
\end{figure}
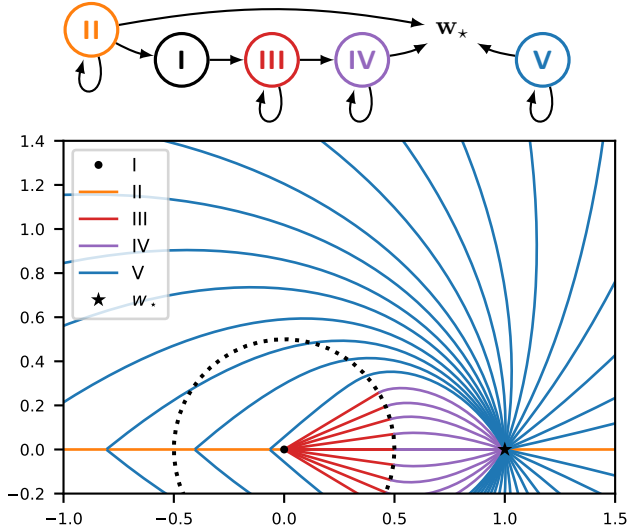

\paragraph{Intrinsic low-dimensional structure of the optimal control solution.} As is hinted in the analysis of Regime~V, the optimal control $\bx$ sometimes lies in the 2D subspace spanned by $\bw$ and $\bp$. In fact, this holds not only for Regime~V but for the whole problem. In particular, we make the following observation.
\begin{thm}\label{thm:2D}
	There always exists a global optimal trajectory of~\eqref{eq:PMP_LS} that lies in a 2D subspace of $\R^n$.
\end{thm}
The detailed proof can be found in the appendix.
An immediate consequence of Theorem~\ref{thm:2D} is that if $\bw_0$ and $\bw_\star$ are linearly independent, we only need to consider trajectories that are confined to the subspace $\mathrm{span}\{\bw_0,\bw_\star\}$. When $\bw_0$ and $\bw_\star$ are aligned, trajectories are still 2D, and any subspace containing $\bw_0$ and $\bw_\star$ is equivalent and arbitrary choice can be made.

This insight is extremely important because it enables us to restrict our attention to 2D trajectories even though the dimensionality of the original problem ($n$) may be huge. This allows us to not only obtain a more elegant and accurate solution in solving the necessary condition induced by PMP, but also to parametrize direct and indirect approaches (see Sections~\ref{sec:indirect} and~\ref{sec:direct}) to solve this intrinsically 2D problem more efficiently.

\begin{figure}[th]
	\includegraphics{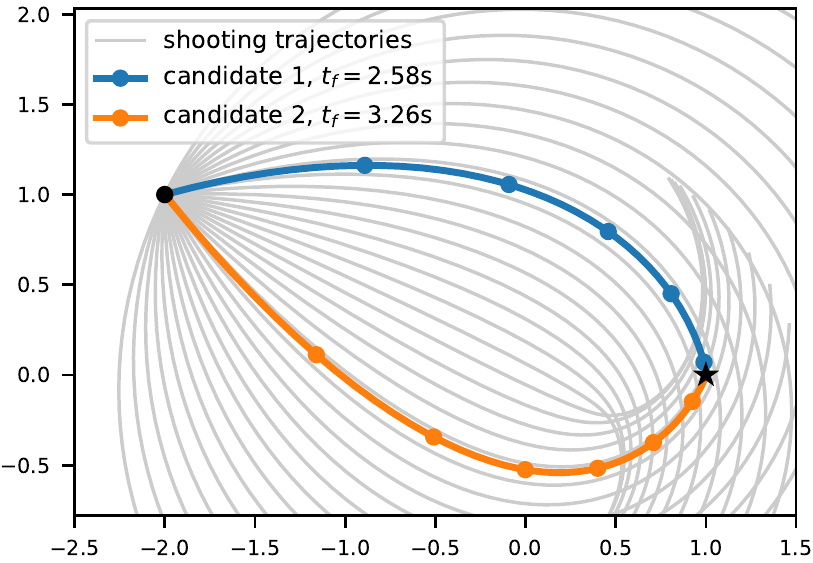}
	\caption{Trajectories found using a shooting approach (Section~\ref{sec:indirect}) with $\bw_0=(-2,1)$ and $\bw_\star=(1,0)$. Gray curves show different shooting trajectories while the blue and orange curves show two trajectories that satisfy the necessary conditions for optimality~\eqref{eq:PMP_LS}. Markers show intervals of $0.5$ seconds, which is roughly 50 steps when using a stepsize of $\eta=0.01$.\label{fig:shooting}}
\end{figure}

\paragraph{Multiplicity of Solution Candidates.} The PMP conditions are only \textit{necessary} for optimality. Therefore, the optimality conditions~\eqref{eq:PMP_LS} need not have a unique solution. We illustrate this phenomenon in Figure~\ref{fig:shooting}. We used a shooting approach (Section~\ref{sec:indirect}) to propagate different choices of $\bp^\star(0)$ forward in time. It turns out two choices lead to trajectories that end at $\bw_\star$, and they do not have equal total times. So in general, PMP identifies \textit{optimal trajectory candidates}, which can be thought of as local minima for this highly nonlinear optimization problem.

\section{NUMERICAL METHODS}
\label{sec:numerical}

While the PMP yields necessary conditions for time-optimal control as detailed in Section~\ref{sec:example}, there is no closed-form solution in general. We now present and discuss four numerical methods: CNLP and NLP are different implementations of time-optimal control, while GREEDY and STRAIGHT are heuristics.

\paragraph{CNLP:} This approach solves the continuous gradient flow limit of the machine teaching problem using a direct approach (Section~\ref{sec:direct}). Specifically, we used the NLOptControl package~\citep{nloptcontrol}, which is an implementation of the $hp$-pseudospectral method GPOPS-II~\citep{GPOPS-II} written in the Julia programming language using the JuMP modeling language~\citep{JuMP} and the IPOPT interior-point solver~\citep{IPOPT}. The main tuning parameters for this software are the integration scheme and the number of mesh points. We selected the trapezoidal integration rule with $100$ mesh points for most simulations. We used CNLP to produce the trajectories in Figures~\ref{fig:optvsgreedy} and~\ref{fig:vectorfield}.

\paragraph{NLP:} A na\"ive approach to optimal control is to find the minimum $T$ for which there is a feasible input sequence to drive the learner to $\bw_\star$.
Fixing $T$, the feasibility subproblem is a nonlinear program over $2T$ $n$-dimensional 
variables $\bx_0, \ldots, \bx_{T-1}$ and $\bw_1, \ldots, \bw_{T}$ constrained by learner dynamics.  Recall 
$\bw_0$ is given,
and one can fix $y_t=R_y$ for all $t$ by Proposition~\ref{prop:y_simplify}.
For our learner~\eqref{eq:regression}, the feasibility problem is
\begin{align}
\min_{\bw_{1:T},\, \bx_{0:T-1}}\qquad & 0 \label{eq:feasibility} \\
\text{s.t.} \qquad
& \bw_T = \bw_\star \nonumber\\
& \bw_{t+1} =\bw_t - \eta (\bw_t^\tp\bx_t - R_y)\bx_t \nonumber\\
& \norm{\bx_t} \le R_x, \quad \forall t=0, \ldots, T-1. \nonumber
\end{align}
As in the CNLP case, we modeled and solved the  subproblems~\eqref{eq:feasibility} using JuMP and IPOPT. We also tried Knitro, a state-of-the-art commercial solver \citep{KNITRO}, and it produced similar results. We stress that such feasibility problems are difficult; IPOPT and Knitro can handle moderately sized $T$.
For our specific learner~\eqref{eq:regression} there are 2D optimal control and state trajectories in $\mathrm{span}\{\bw_0,\bw_\star\}$ as discussed in Section~\ref{sec:example}. Therefore, we reparameterized~\eqref{eq:feasibility} to work in 2D.

On top of this, we run a binary search over positive integers to find the minimum $T$ for which the subproblem~\eqref{eq:feasibility} is feasible.
Subject to solver numerical stability, the minimum $T$ and its feasibility solution $\bx_0, \ldots, \bx_{T-1}$ is the time-optimal control.
While NLP is conceptually simple and correct, it requires solving many subproblems with $2T$ variables and $2T$ constraints, making it less stable and scalable than CNLP.

\paragraph{GREEDY:}
We restate the greedy control policy initially proposed by \cite{liu2017iterative}. It has the advantage of being computationally more efficient and readily applicable to different learning algorithms (i.e. dynamics).
Specifically for the least squares learner~\eqref{eq:regression} and given the current state $\bw_t$, GREEDY solves the following optimization problem to determine the next teaching example $(\bx_t,y_t)$: 
\begin{align}
\min_{(\bx_t,y_t)\in \mathcal{U}} \qquad &\norm{\bw_{t+1} -\bw_\star}^2\\
\mbox{s.t.} \qquad & \bw_{t+1} = \bw_t - \eta (\bw_t^\tp\bx_t - y_t) \bx_t. \nonumber
\end{align}
The procedure repeats until $\bw_{t+1} = \bw_\star$. We used the \MATLAB function \texttt{fmincon} to solve the above quadratic program iteratively.
We point out that the optimization problem is not convex.
Moreover, $\bw_{t+1}$ does not necessarily point in the direction of $\bw_\star$. This is evident in Figure~\ref{fig:optvsgreedy} and Figure~\ref{fig:pacman}. 

\paragraph{STRAIGHT:}
We describe an intuitive control policy: at each step, move $\bw$ in straight line toward $\bw_\star$ as far as possible subject to the constraint $\mathcal{U}$.
This policy is less greedy than GREEDY because it may not reduce $\norm{\bw_{t+1} -\bw_\star}^2$ as much at each step.
The per-step optimization in $\bx$ is a 1D line search:
\begin{align}\label{eq:straight}
\min_{a,y_t \in \R} \qquad&\norm{\bw_{t+1} -\bw_\star}^2\\
\mbox{s.t.} \qquad & \bx_t = a (\bw_\star - \bw_t)/\|\bw_\star - \bw_t\| \nonumber \\
& (\bx_t,y_t)\in \mathcal{U} \nonumber \\
& \bw_{t+1} = \bw_t - \eta (\bw_t^\tp\bx_t - y_t) \bx_t. \nonumber
\end{align}
The line search~\eqref{eq:straight} can be solved in closed-form. In particular, one can obtain that
\begin{equation*}
a= 
\begin{cases}
\min\{R_x,\frac{R_y\|\bw_\star-\bw\|}{2(\bw_\star-\bw)^\tp\bw}\},& \text{if } (\bw_\star-\bw)^\tp\bw>0\\
R_x,              & \text{otherwise.}
\end{cases}
\end{equation*}

\subsection{Comparison of Methods}

\begin{figure*}[tb]
	\includegraphics{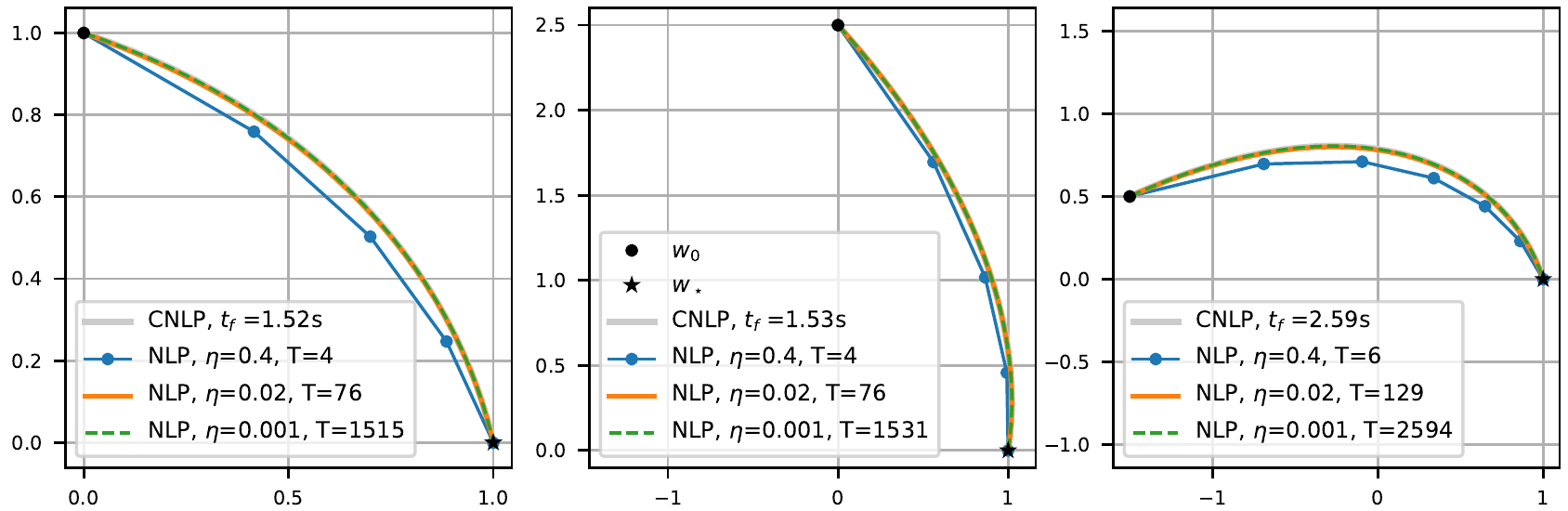}\vspace{-2mm}
	\caption{Comparison of CNLP vs NLP. All teaching tasks use the terminal point $\bw_\star=(1,0)$. The initial points used are 
		$\bw_0 = (0,1)$ (left panel), $\bw_0 = (0,2.5)$ (middle panel), and $\bw_0 = (-1.5,0.5)$ (right panel). We observe that the NLP trajectories on learners with smaller $\eta$'s quickly converges to the CNLP trajectory.\label{fig:convergence}}\vspace{-2mm}
\end{figure*}

We ran a number of experiments to study the behavior of these numerical methods.
In all experiments, the learner is gradient descent on least squares~\eqref{eq:regression}, and the control constraint set is $\norm{\bx}\le 1, |y|\le 1$.
Our first observation is that CNLP has a number of advantages:
\begin{enumerate}
	\item CNLP's continuous optimal state trajectory matches NLP's discrete state trajectories, especially on learners with small $\eta$.
	This is expected, since the continuous optimal control problem is obtained asymptotically from the discrete one as $\eta \rightarrow 0$.
	Figure~\ref{fig:convergence} shows the teaching task $\bw_0=(1, 0) \Rightarrow \bw_* = (1,0)$.  Here we compare CNLP with  NLP's optimal state trajectories on four gradient descent learners with different $\eta$ values.
	The NLP optimal teaching sequences vary drastically in length $T$, but their state trajectories quickly overlap with CNLP's optimal trajectory.
	
	\item CNLP is quick to compute, while NLP runtime grows as the learner's $\eta$ decreases.
	Table~\ref{tab:dataset} presents the wall clock time.  With a small $\eta$, the optimal control takes more steps (larger $T$).  Consequently, NLP must solve a nonlinear program with more variables and constraints.  In contrast, CNLP's runtime does not depend on $\eta$.
	
	\item CNLP can be used to approximately compute the ``teaching dimension'', i.e. the minimum number of sequential teaching steps $T$ for the discrete problem.
	Recall CNLP produces an optimal terminal time $t_f$.  When the learner's $\eta$ is small, the discrete ``teaching dimension'' $T$ is related by $T \approx t_f / \eta$.
	This is also supported by Table~\ref{tab:dataset}.
	
\end{enumerate}
That said, it is not trivial to extract a discrete control sequence from CNLP's continuous control function.
This hinders CNLP's utility as an optimal teacher.

\begin{table}[ht]
	\caption{Teaching sequence length and wall clock time comparison. NLP teaches three learners with different $\eta$'s.  Target is always $\bw_\star=(1,0)$. All experiments were performed on a conventional laptop.}
	\vspace{1mm}
	\label{tab:dataset}
	\begin{tabularx}{\columnwidth}{ c | r r r | r }
		\toprule
		&\multicolumn{3}{c|}{\textbf{NLP}} & \textbf{CNLP} \\
		$\bw_0$ & $\eta=0.4$ & 0.02 & 0.001 &  \\
		\midrule
		$(0,1)$ & $T=3$ & 75 & 1499 & $t_f=1.52\mathrm{s}$\\
		&  0.013s & 0.14s & 59.37s & 4.1s\\
		\hline
		$(0,2.5)$ & $T=5$ & 76 & 1519 & $t_f=1.53\mathrm{s}$\\
		&  0.008s & 0.11s & 53.28s & 2.37s\\
		\hline
		$(-1.5,0.5)$ & $T=6$ & 128 & 2570 & $t_f=2.59\mathrm{s}$\\
		&  0.012s & 0.63s & 310.08s & 2.11s\\
		\hline	
	\end{tabularx}
	\vspace{-2mm}
\end{table}

\begin{table}[ht]
	\caption{Comparison of teaching sequence length $T$. We fixed $\eta=0.01$ in all cases.}
	\vspace{1mm}
	\label{tab:Tcompare}
	\begin{tabularx}{\columnwidth}{ c c | c  c  c }
		\toprule
		$\bw_0$ & $\bw_\star$ &\textbf{NLP} & \!\!\textbf{STRAIGHT}\!\! & \textbf{GREEDY} \\
		\midrule
		$(0,1)$ & $(2,0)$ & 148 & 161 & 233 \\
		$(0,2)$ & $(4,0)$ & 221 & 330 & 721 \\
		$(0,4)$ & $(8,0)$ & 292 & 867 & 2667 \\
		$(0,8)$ & $(16,0)$ & 346 & 2849 & 10581 \\
		\hline
	\end{tabularx}
	\vspace{-2mm}
\end{table}

\begin{figure}[th]
	\includegraphics{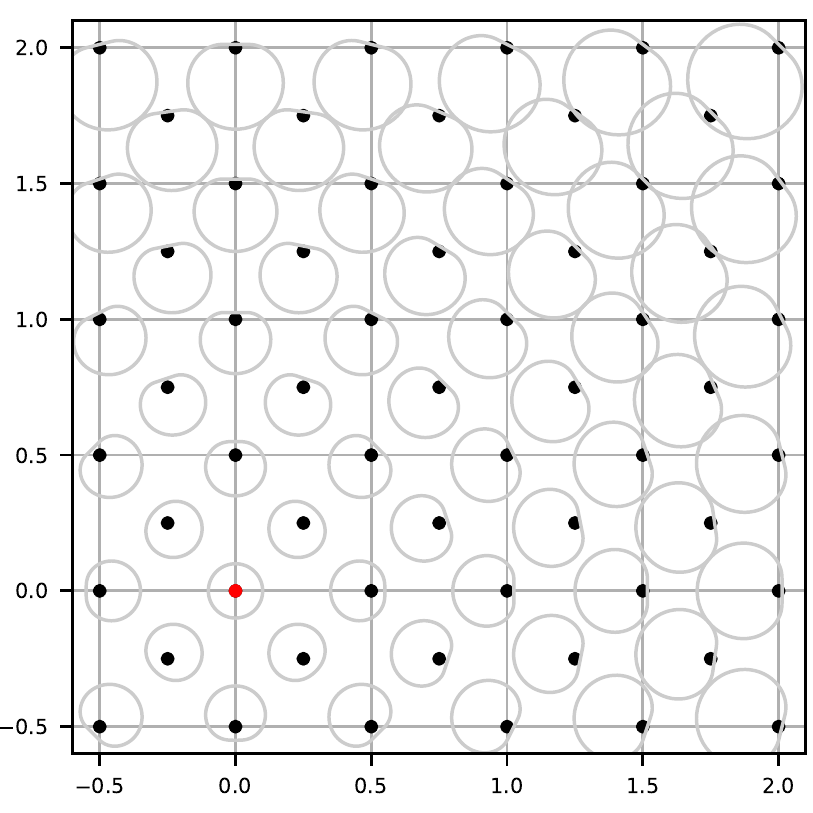}
	\vspace{-6mm}
	\caption{Points reachable in one step of gradient descent (with $\eta = 0.1$) on a least-squares objective starting from each of the black dots. 
		There is circular symmetry about the origin (red dot).
		\label{fig:pacman_grid}}
\end{figure}

\begin{figure}[ht]
	\includegraphics{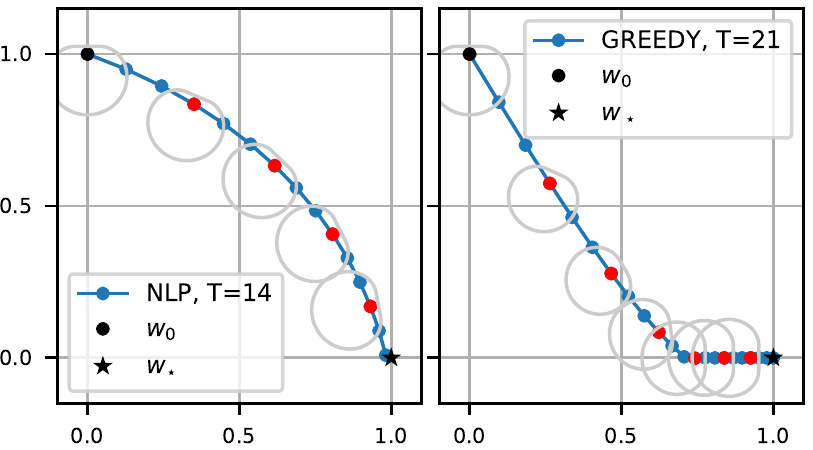}
	\vspace{-3mm}
	\caption{Reachable sets along the trajectory of NLP (left panel) and GREEDY (right panel). To minimize clutter, we only show every $3^\text{rd}$ reachable set. For this simulation, we used $\eta=0.1$. The greedy approach makes fast progress initially, but slows down later on.
		\label{fig:pacman}}
\end{figure}

\begin{figure}[!h]
	\vspace{-0mm}
	\includegraphics{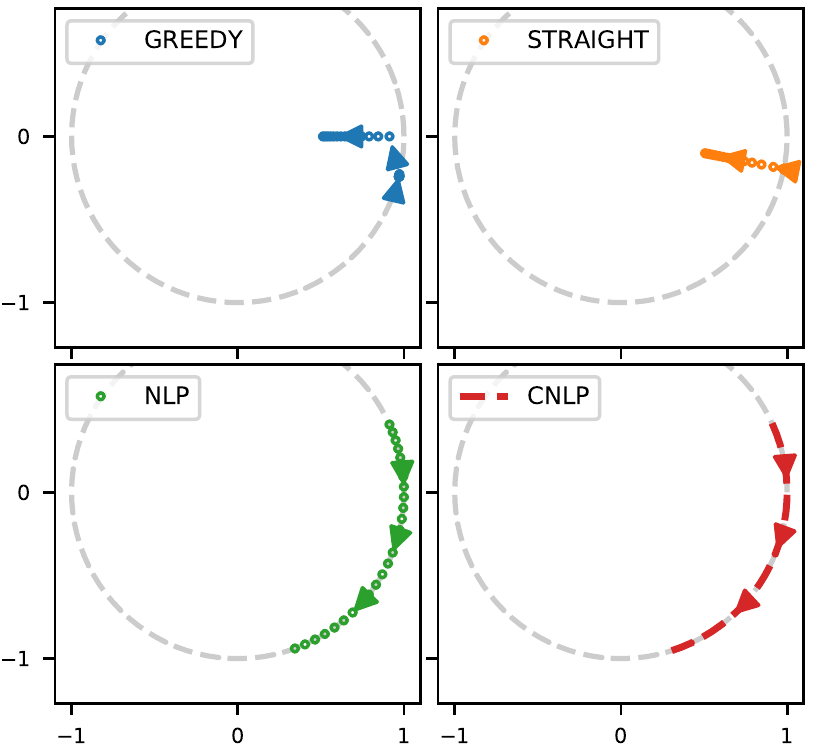}
	\vspace{-6mm}
	\caption{Trajectories of the input sequence $\{\bx_t\}$ for GREEDY, STRAIGHT, and NLP methods and the corresponding $\bx(t)$ for CNLP. The teaching task is $\bw_0 = (-1.5, 0.5)$, $\bw_\star = (1,0)$, and $\eta=0.01$. Markers show every 10 steps. Input constraint is $\norm{\bx}\le 1$.\label{fig:control}}
\end{figure}

Our second observation is that NLP, being the discrete-time optimal control, produces shorter teaching sequences than GREEDY or STRAIGHT.
This is not surprising, and we have already presented three teaching tasks in Figure~\ref{fig:optvsgreedy} where NLP has the smallest $T$. In fact, there exist teaching tasks on which GREEDY and STRAIGHT can perform arbitrarily worse than the optimal teaching sequence found by NLP. A case study is presented in Table~\ref{tab:Tcompare}. In this set of experiments, we set $\bw_0 = (a,0)$ and $\bw_\star = (0,2a)$. As $a$ increases, the ratio of teaching sequence length between STRAIGHT and NLP and between GREEDY and NLP grow at an exponential rate.

We now dig deeper and present an intuitive explanation of why GREEDY requires more teaching steps than NLP. The fundamental issue is the nonlinearity of the learner dynamics~\eqref{eq:regression} in $\bx$.
For any $\bw$ let us define the one-step reachable set $\set{\bw - \eta (\bw^\tp \bx - y)\bx}{(\bx, y) \in \mathcal U }$.
Figure~\ref{fig:pacman_grid} shows a sample of such reachable sets.
The key observation is that the starting $\bw$ is quite close to the boundary of most reachable sets. In other words, there is often a compressed direction---from $\bw$ to the closest boundary of $\mathcal U$---along which $\bw$ makes minimal progress. The GREEDY scheme falls victim to this phenomenon.

Figure~\ref{fig:pacman} compares NLP and GREEDY on a teaching task chosen to have short teaching sequences in order to minimize clutter. GREEDY starts by eagerly descending a slope and indeed this quickly brings it closer to $\bw_\star$.
Unfortunately, it also arrived at the $x$-axis.  For $\bw$ on the $x$-axis, the compressed direction is horizontally outward.
Therefore, subsequent GREEDY moves are relatively short, leading to a large number of steps to reach $\bw_\star$. Interestingly, STRAIGHT is often better than GREEDY because it also avoids the $x$-axis compressed direction for general $\bw_0$.

We illustrate the optimal inputs in Figure~\ref{fig:control}, which compares $\{\bx_t\}$ produced by STRAIGHT, GREEDY, and NLP and the $\bx(t)$ produced by CNLP. The heuristic approaches eventually take smaller-magnitude steps as they approach $\bw_\star$ while NLP and CNLP maintain a maximal input norm the whole way.

\section{CONCLUDING REMARKS}

Techniques from optimal control are under-utilized in machine teaching, yet they have the power to provide better quality solutions as well as useful insight into their structure.

As seen in Section~\ref{sec:example}, optimal trajectories for the least squares learner are fundamentally 2D. Moreover, there is a taxonomy of regimes that dictates their behavior.
We also saw in Section~\ref{sec:numerical} that the continuous CNLP solver can provide a good approximation to the true discrete trajectory when $\eta$ is small. CNLP is also more scalable than simply solving the discrete NLP directly because NLP becomes computationally intractable as $T$ gets large (or $\eta$ gets small), whereas the runtime of CNLP is independent of $\eta$.

A drawback of both NLP and CNLP is that they produce \textit{trajectories} rather than \textit{policies}. In practice, using an open-loop teaching sequence $(\bx_t,y_t)$ will not yield the $\bw_t$ we expect due to the accumulation of small numerical errors as we iterate. In order to find a control policy, which is a map from state $\bw_t$ to input $(\bx_t,y_t)$, we discussed the possibility of solving HJB (Section~\ref{sec:dynprog}) which is computationally expensive.

An alternative to solving HJB is to pre-compute the desired trajectory via CNLP and then use \textit{model-predictive control} (MPC) to find a policy that tracks the reference trajectory as closely as possible. Such an approach is used in~\citet{MPC}, for example, to design controllers for autonomous race cars, and would be an interesting avenue of future work for the machine teaching problem.

Finally, this paper presents only a glimpse at what is possible using optimal control. For example, the PMP is not restricted to merely solving time-optimal control problems. It is possible to analyze problems with state- and input-dependent running costs, state and input pointwise or integral constraints, conditional constraints, and even problems where the goal is to reach a target \textit{set} rather than a target point.


\newpage
\bibliographystyle{abbrvnat}
\bibliography{teaching_control,z}

\newpage
\section{Appendix}

\paragraph{Proof of modified Proposition~\ref{prop:y_simplify}.}

In this version, we assume $(\bw,\bx,y)$ is a trajectory of~\eqref{eq:regression} rather than being a trajectory of~\eqref{eq:PMP_LS}.

All we need to show is that for any pair of $(\bx,y)$, there exist another pair $(\tilde\bx,R_y)$, such that they give the same update. In particular, we set $\tilde\bx=a\bx$ and show that there always exists an $a\in[-1,1]$ such that
\begin{equation*}
(y-\bw^\tp\bx)\bx = (R_y-\bw^\tp a\bx)a\bx.
\end{equation*}
This simplifies to
\begin{equation}\label{a2}
g(a) \defeq (\bw^\tp\bx) a^2 - R_y a +(y-\bw^\tp\bx)=0.
\end{equation}
The discriminant of the quadratic~\eqref{a2} is
\begin{align*}
R_y^2 - 4\bw^\tp\bx(y-\bw^\tp\bx)
&\ge R_y^2 - 4|\bw^\tp\bx|\left( R_y + |\bw^\tp\bx| \right) \\
&= \left( R_y - 2|\bw^\tp \bx| \right)^2 \ge 0
\end{align*}
So there always exists a solution $a\in\R$. Moreover, $g(-1) = R_y+y\geq 0 $ and $g(1) = -R_y+y\leq 0$, so there must be a real root in $[-1,1]$. \qedhere

\paragraph{Proof of Theorem~\ref{thm:2D}.} We showed in Section~\ref{sec:example} that Regime~V trajectories are 2D. We also argued that solutions that reach $\bw_\star$ via Regime~III--IV are not unique and need not be 2D. We will now show that it's always possible to construct a 2D solution.

We begin by characterizing the set of $\bw_\star$ reachable via Regime~III--IV. Recall from Section~\ref{sec:example} that the transition between III and IV occurs when $\norm{\bw} = R \defeq \tfrac{R_y}{2R_x}$. If $t_0$ is the time at which this transition occurs, then for $0 \le t \le t_0$, the solution is $\bx = \frac{R_x}{\norm{\bw}}\bw$, which leads to a straight-line trajectory from $\bw_0$ to $\bw(t_0)$.

Now consider the part of the trajectory in Regime~IV, where $t_0 \le t \le t_f$. As derived in Section~\ref{sec:example}, Regime~IV trajectories satisfy $\dot\bw = \bw^\tp \bx = \tfrac{R_y}{2}$. These lead to $\frac{\mathrm{d}\norm{\bw}^2}{\mathrm{d}t} = \frac{R_y^2}{2}$, which means that $\norm{\bw}$ grows at the same rate regardless of $\bx$. If our trajectory reaches $\bw(t_f) = \bw_\star$, then we can deduce via integration that
\begin{equation}\label{tf}
\norm{\bw_\star}^2 - \norm{\bw(t_0)}^2 = \tfrac{R_y^2}{2}(t_f-t_0),
\end{equation}
Suppose $(\bw(t),\bx(t))$ for $t_0 \le t \le t_f$ is a trajectory that reaches $\bw_\star$. Refer to Figure~\ref{fig:reachable}. The reachable set at time $t_f$ is a spherical sector whose boundary requires a trajectory that maximizes curvature. We will now derive this fact.

\begin{figure}[ht]
\centering
\begin{tikzpicture}[very thick,>=latex]
\def\t{25} 
\def\RA{1} 
\def\R{2}  
\def\RR{4.5} 
\def\RRA{16} 
\def\FR{2.68} 
\def\FA{32} 
\def\HFA{16} 
\def\RAtmp{19} 

\draw[black!20] (0,0) -- ({5*cos(\RRA)},{5*sin(\RRA)});

\coordinate (W1) at ({\R*cos(\t)},{\R*sin(\t)});
\draw[C3] (0,0) -- node[pos=0.8,anchor=south east]{\footnotesize{\textsf{\textbf{III}}}} (W1);

\coordinate (P1) at ($ (W1)+ ({\FR*cos(\t-\FA)},{\FR*sin(\t-\FA)}) $);
\coordinate (P2) at ($ (W1) + ({\FR*cos(\t+\FA)},{\FR*sin(\t+\FA)}) $);
\path[fill=C4!30, name path = lowcurve] (W1) to [bend left = \FA] (P1)
	to [bend right= \HFA] (P2)
	to [bend left =\FA] (W1) -- cycle;

\draw[C4,densely dotted] (W1) to [bend left = \FA] (P1);
\draw[C4,densely dotted] (W1) to [bend right = \FA]
	node[pos=0.5,anchor=south east]{\footnotesize{\textsf{\textbf{IV}}}} (P2);

\draw[C4!80,densely dotted] (P1) to [bend left = 6] (P2);
\draw[C4!50,densely dotted] (P1) to [bend right = 6] (P2);

\coordinate (Wstar) at ({\RR*cos(\RRA)},{\RR*sin(\RRA)});
\path[name path = radial] (0,0) to (Wstar);
\path [name intersections={of=lowcurve and radial,by=W2}];
\draw[C4] (W1) to [bend left = \RAtmp] (W2) -- (Wstar);
\node[anchor=north,shift={(-0.1,0)}] at (W2) {\footnotesize $\bw(t_1)$};

\draw[black!20] (\RR,0) arc (0:50:\RR);
\draw[dotted] (\R,0) arc (0:90:\R);
\draw[<->] (0,3.5) -- (0,0) -- node[anchor=north east,pos=0]{0} (6,0);
\node[anchor=north] at (\R,0) {$R$};
\node[anchor=north] at (\RR,0) {$\norm{\bw_\star}$};

\coordinate (W0) at ({\RA*cos(\t)},{\RA*sin(\t)});
\node[circle,fill=black,inner sep=1.2] at (W0) {};
\node[anchor=south east,shift={(0,-0.1)}] at (W0) {\footnotesize $\bw_0$};
\node[anchor=north west,shift={(0,0.18)}] at (W1) {\footnotesize $\bw(t_0)$};

\node[fill=black,star,star points=5,star point ratio=3,inner sep=0.8] at (Wstar) {};
\node[anchor=north west,shift={(0,0.15)}] at (Wstar) {\footnotesize $\bw(t_f) = \bw_\star$};

\end{tikzpicture}

\vspace{-3mm}
\caption{If a reachable $\bw_\star$ is contained in the concave funnel shape, which is the reachable set in Regime~IV, it can be reached by some trajectory $(\bw(t),\bx(t))$ lying entirely in the 2D subspace defined by $\mathrm{span}\{\bw_0,\bw_\star\}$: follow the max-curvature solution until $t_1$ and then transition to a radial solution until $t_f$.\label{fig:reachable}\vspace{-2mm}}
\end{figure}
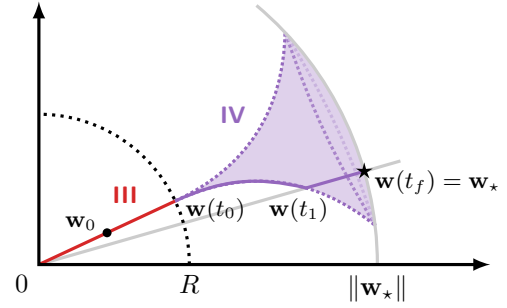

Let $\theta_{\max}$ be the largest possible angle between $\bw(t_0)$ and any reachable $\bw(t_f) = \bw_\star$, where we have fixed $t_f$. Define $\theta(t)$ to be the angle between $\bw(t)$ and $\bw(t_f)$.
\[
\theta(t_0) \,=\, \int_{t_0}^{t_f} \dot\theta\,\mathrm{d}t
\,\le\, 
\int_{t_0}^{t_f} | \dot\theta |\,\mathrm{d}t
\]
An alternative expression for this rate of change is the projection of $\dot\bw$ onto the orthogonal complement of $\bw$:
\begin{align*}
|\dot\theta| &= \frac{\normm{ \dot{\bw}-\bl(\dot{\bw}^\tp \tfrac{\bw}{\|\bw\|}\br)\tfrac{\bw}{\|\bw\|} }}{\norm{\bw}} 
= \frac{R_y\normm{ \bx-\tfrac{R_y}{2\|\bw\|^2}\bw }}{2\norm{\bw}}
\end{align*}
Where we used the fact that $\dot\bw = \bw^\tp \bx = \tfrac{R_y}{2}$ in Regime~IV. Now,
\begin{align}
\theta_{\max} &= \max_{\substack{\bx:\,  \bw^\tp\bx=R_y/2 \\ \norm{\bx} \le R_x}} \theta(t_0) \notag\\
&\le \max_{\substack{\bx:\,  \bw^\tp\bx=R_y/2 \\ \norm{\bx} \le R_x}} \int_{t_0}^{t_f} \frac{R_y\normm{ \bx-\tfrac{R_y}{2\|\bw\|^2}\bw }}{2\norm{\bw}}\,\mathrm{d}t\nonumber\\
&\leq \int_{t_0}^{t_f} \frac{\sqrt{R_x^2-\bl(\frac{R_y}{2\|\bw\|}\br)^2}}{\|\bw\|}\,\mathrm{d}t
\label{a4}
\end{align}
In the final step, we maximized over $\bx$. Notice that the integrand~\eqref{a4} is an  upper bound that only depends on $t_0$ and $\norm{\bw_\star}$ but not on $\bx$. One can also verify that this upper bound is achieved by the choice
\begin{equation*}
\bx = \frac{R_y}{2\|\bw\|}\hat\bw + \sqrt{R_x^2-\left(\frac{R_y}{2\|\bw\|}\right)^2}\frac{\bw_\star-(\hat\bw^\tp\bw_\star)\hat\bw}{\norm{\bw_\star-(\hat\bw^\tp\bw_\star)\hat\bw}}.
\end{equation*}
where $\hat\bw \defeq \bw/\norm{\bw}$ and $\bw_\star$ is any vector that  satisfies~\eqref{tf} with angle $\theta_{\max}$ with $\bw(t_0)$. Any $\bw_\star$ with this norm but angle $\theta_f < \theta_{\max}$ can also be reached by using the max-curvature control until time $t_1$, where $t_1$ is chosen such that $\theta_{f} = \int_{t_0}^{t_1} \frac{\sqrt{R_x^2-\bl(\frac{R_y}{2\|\bw\|}\br)^2}}{\|\bw\|}\,\mathrm{d}t$, and then using $\bx = \frac{R_y}{2\|\bw\|^2}\bw$ for $t_1 \le t \le t_f$. This piecewise path is illustrated in Figure~\ref{fig:reachable}.

Our constructed optimal trajectory lies in the 2D span of $\bw_\star$ and $\bw_0$. This shows that all reachable $\bw_\star$ can be reached via a 2D trajectory. \qedhere

\end{document}